\pdfoutput=1

\documentclass[11pt]{article}

\usepackage[preprint]{acl}

\usepackage{times}
\usepackage{latexsym}
\usepackage{amsmath}
\usepackage{amssymb}
\usepackage{mathtools}
\usepackage{amsthm}
\usepackage{float}
\usepackage{multirow}
\usepackage[T1]{fontenc}
\usepackage{todonotes}
\usepackage[utf8]{inputenc}
\usepackage{microtype}

\usepackage{inconsolata}

\usepackage{graphicx}
\usepackage{subcaption}

\usepackage{tcolorbox}
\tcbuselibrary{skins}
\newtcolorbox{promptbox}[1][]{
    enhanced,
    colback=blue!3!white,
    colframe=blue!30!black,
    boxrule=0.4pt,
    arc=1mm,
    left=4pt, right=4pt, top=4pt, bottom=4pt,
    fontupper=\small,
    #1
}
\usepackage{enumitem}

\usepackage{tikz}
\usetikzlibrary{positioning, shapes.geometric, arrows.meta, fit, backgrounds, calc}

\theoremstyle{plain}
\newtheorem{theorem}{Theorem}[section]

\newtheorem{lemma}[theorem]{Lemma}

\theoremstyle{definition}
\newtheorem{definition}[theorem]{Definition}

\theoremstyle{remark}

\title{A theoretical model for task routing in mixture-of-expert transformers}

\author{
  Vinoth Nandakumar\thanks{\ \ Correspondence to: \texttt{vinoth.90@gmail.com}} \\
  University of Sydney
  \And
  Yongli Xiang \\
  University of Sydney
  \AND
  Yunzhi Yao \\
  Zhejiang University
  \And
  Peike Li \\
  Google Research
  \And
  Tongliang Liu \\
  University of Sydney
}

\begin{document}
\maketitle
\begin{abstract}
Mixture-of-experts (MoE) layers enable the scaling of transformer models while keeping the inference compute fixed. While task-expert specialization has been observed in empirical studies of frontier MoE transformer models, existing theoretical work analyzes this using continuous mixture models that cannot be used to model natural language effectively. An important open question is to \textit{theoretically explain task-expert specialization in transformer MoE models using discrete models of language}. To address this, we represent structured knowledge via syntactic templates and finite key-value dictionaries, and prove formally that a single-layer MoE transformer can encode knowledge by using experts that specialize in the corresponding tasks. Our construction shows how queries are routed to unique, task-specific experts whose size depends solely on the intrinsic complexity of the given task (i.e. the combined size of its syntactic templates and factual dictionary). Our construction provides a theoretical support for empirical results on localized knowledge circuits in MoE models. We support our theoretical findings with experiments evaluating model performance under varying MoE loss functions. \end{abstract}

\section{Introduction}


Modern transformer models achieve state-of-the-art accuracy by leveraging hundreds of attention heads and billions of parameters, yet mounting evidence shows that only a small fraction of these computations is required for any single input \cite{dejavu, circuit-sparsity}. Frontier large language models increasingly adopt mixture-of-experts transformers \cite{deepseek-moe, qwen2-moe}, which dynamically route tokens through a sparse collection of expert modules at each layer, achieving state-of-the-art performance with substantially reduced inference cost \cite{shazeer17, switch22}. 

Empirical studies have observed task–expert specialization in MoE transformers \cite{MoE-knowledge}, where experts learn to specialize in fine-grained semantic tasks. Ideally, each task would be routed to a well-defined sparse subset of experts, whose size scales with the intrinsic complexity of that task, leading to inference-time speedups \cite{taskMoE}. Mapping specific tasks to distinct sub-networks could also enhance interpretability by allowing practitioners to trace decision-making using these functional circuits \cite{modular-LLMs}, thereby leading to safer and more reliable systems. An important open question is whether we can establish a rigorous \textbf{theoretical foundation for explicit task-expert specialization in MoE language models} \cite{Mod-Squad}. 

From a theoretical perspective, recent studies have analyzed task-expert specialization in mixture-of-experts architectures \cite{MoE-theory, routing}, and establish theoretical guarantees by using continuous data distributions, such as Gaussian mixture models. However, these frameworks analyze feedforward-only MoE models lacking multi-head attention. They don't extend to discrete data distributions, such as $n$-gram templates \cite{transformer-ngrams, SSM-ngrams}, which use a predefined pattern with fixed tokens and wildcard symbols to generate sentences. To address these gaps in the theoretical foundations of MoE transformer models, we \textbf{theoretically explain how they can utilize attention heads and task-specialized experts to process and store discrete structures that can be used to model language}. 

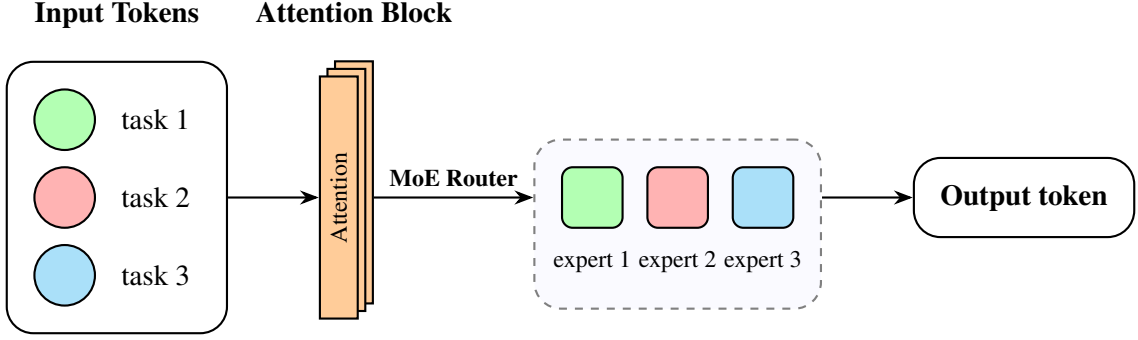
\begin{figure*}[t]
    \centering
    \begin{tikzpicture}[
        node distance=1cm and 1.2cm,
        colored_sq/.style={rectangle, draw=black, thick, rounded corners=4pt, minimum size=0.8cm},
        attn_layer/.style={rectangle, draw=black, thick, fill=orange!40, minimum width=0.5cm, minimum height=3.2cm},
        arrow/.style={-{Stealth[length=2.5mm]}, thick},
        colored_ci/.style={circle, draw=black, thick, rounded corners=4pt, minimum size=0.8cm},
    ]

    \node[colored_ci, fill=green!30] (tok1) {};
    \node[right=0.2cm of tok1] (lbl1) {task 1};

    \node[colored_ci, fill=red!30, below=0.2cm of tok1] (tok2) {};
    \node[right=0.2cm of tok2] (lbl2) {task 2};

    \node[colored_ci, fill=cyan!30, below=0.2cm of tok2] (tok3) {};
    \node[right=0.2cm of tok3] (lbl3) {task 3};

    \begin{scope}[on background layer]
        \node (input_box) [draw=black, thick, rounded corners=10pt, fit=(tok1) (lbl1) (tok3) (lbl3), inner sep=10pt] {};
        \node (input_lbl) [above=0.3cm of input_box.north] {\textbf{Input Tokens}};
    \end{scope}

    \node[attn_layer, right=1.2cm of input_box, xshift=0.2cm, yshift=0.2cm] (attn_back) {};
    \node[attn_layer, right=1.2cm of input_box, xshift=0.1cm, yshift=0.1cm] (attn_mid) {};
    \node[attn_layer, right=1.2cm of input_box] (attn_front) {\rotatebox{90}{\small Attention}};
    
    \node [anchor=base] at (attn_back.north |- input_lbl.base) {\textbf{Attention Block}};

    \node[colored_sq, fill=red!30, right=3.8cm of attn_front] (exp2) {};
    \node[below=0.15cm of exp2] (exp2_lbl) {\small expert 2};

    \node[colored_sq, fill=green!30, left=0.3cm of exp2] (exp1) {};
    \node[below=0.15cm of exp1] (exp1_lbl) {\small expert 1};

    \node[colored_sq, fill=cyan!30, right=0.3cm of exp2] (exp3) {};
    \node[below=0.15cm of exp3] (exp3_lbl) {\small expert 3};

    \begin{scope}[on background layer]
        \node (moe_box) [draw=black!50, dashed, thick, rounded corners=10pt, fit=(exp1) (exp3) (exp2_lbl), inner sep=10pt, fill=blue!2] {};
    \end{scope}

    \node (output_box) [draw=black, thick, rounded corners=10pt, inner sep=10pt, anchor=west] at ([xshift=1.2cm]moe_box.east |- input_box.east) {\textbf{Output token}};

    \draw [arrow] (input_box.east) -- (attn_front.west |- input_box.east);
    
    \draw [arrow] (attn_back.east |- input_box.east) -- node[above, font=\small\bfseries] {MoE Router} (moe_box.west |- input_box.east);
    
    \draw [arrow] (moe_box.east |- input_box.east) -- (output_box.west);

    \end{tikzpicture}
    
    \caption{An architectural overview of task routing in mixture-of-expert transformers. The attention block separates dictionary tokens from templates, allowing the router to use template structure to select a task-specific expert for prediction. Expert size scales additively with task complexity.}
    \label{fig:moe_theory_overview}
    \vspace{-4mm}
\end{figure*}


\textbf{Our contributions.} To study this question theoretically, we introduce a simplified discrete model of structured knowledge based on syntactic templates and finite key–value dictionaries. Within this framework, we establish a formal result showing that a single–layer \emph{mixture-of-experts} transformer can represent structured knowledge using task-specialized experts. We theoretically construct MoE transformers in which the attention mechanism disentangles template structure from the factual subject,  while the routing mechanism maps inputs to task-specific experts. Each expert then performs task-specific associative retrieval over the corresponding knowledge dictionary. In particular, our key contributions are as follows. 


\begin{itemize}[leftmargin=*, topsep=0pt]\setlength{\parskip}{0pt}
    \item We identify a gap in existing MoE theory, which primarily explains routing using mixture models that cannot model natural language effectively. To address this, we propose a theoretical framework for sparse routing in MoE transformers based on syntactic templates and finite key–value dictionaries. 
    \item Within this framework, we show that attention blocks can separate structural templates from factual information, enabling routing based on relational structure while sparse experts perform task-specific associative retrieval, and the size of the expert scales with task complexity. 
    \item Using synthetic knowledge datasets, we empirically study routing behavior under different training objectives and observe that explicit task-aware routing objectives produce substantially clearer task–expert alignment than standard objectives alone. 
\end{itemize}

\noindent The remainder of this paper is organized as follows. Section~\ref{related-work} reviews related work on MoE transformers and knowledge circuits. Section~\ref{preliminaries} introduces a synthetic model for knowledge data, and transformer architectures that are used in subsequent sections. Section~\ref{theorems} presents our key theoretical results, explicitly constructing dense transformers with sparse circuits  (Theorem~\ref{sparse-circuits}) and MoE transformers with task-specific 
experts (Theorem~\ref{MoE}). Section~\ref{proofs} gives an overview of the proofs for these core theorems, and in Section \ref{experiments} we conduct experiments on synthetic data that support our key theoretical findings. 

\section{Related work} \label{related-work}

\subsection{Theoretical models for mixture-of-expert transformers}

To bridge the gap between empirical success and mathematical understanding, recent theoretical studies have sought to formally characterize the optimization and generalization properties of mixture-of-experts architectures under simplified settings. For instance, \citet{MoE-theory} provides a foundational analysis of the MoE model trained with gradient descent where each expert is a two-layer CNN, using a data distribution with cluster structure. Building upon this, \citet{routing} uses synthetic data generated from a Gaussian mixture model, and shows that the router in a MoE model with a feedforward network learns the latent cluster structures. Further investigations have expanded on these properties by establishing formal generalization error bounds for sparse MoEs \cite{MoE-generalization} and demonstrating their theoretical advantages in mitigating catastrophic forgetting during continual learning \cite{MoE-continual}. However, a significant limitation of these existing frameworks is that they do not use MoE layers with attention mechanisms, nor do they model the discrete, sequential data typical of natural language processing tasks. 

\subsection{Task-specific experts for language models} In \citet{demix}, the authors introduce DeMix Layers, a modular mixture-of-experts (MoE) architecture that embeds domain specialization directly within the transformer network. Each layer contains a mixture of feed-forward experts, where every expert corresponds to a specific domain, and attention layers are shared across domains. In \cite{scale-experts, BTM}, the authors propose an expert-style training recipe that discovers domains without supervision, clusters the corpus accordingly, trains a separate expert LM per cluster, and combines them as a sparse ensemble at inference (so only a few experts are active per input). While these studies highlight the practical success of task-specific routing, they primarily offer empirical observations without formal guarantees. Our results bridge this gap by providing a mathematically rigorous foundation for these architectures. 

\subsection{Knowledge circuits}

Recent investigations into how transformers store and retrieve facts have identified `knowledge circuits' \cite{knowledge-circuits, wang2025headmap}, which are localized sub-networks within pretrained models dedicated to specific factual associations \cite{factual-recall}. Building on this with mechanistic interpretability, \citet{knowledge-acquisition} utilizes circuits to analyze how large language models acquire and update information during continual pre-training. Our results complement recent theoretical results understanding factual recall with dense transformers \cite{nichani24}, by showing how mixture-of-expert models can store key-value information using sparse, task-specific circuits. Our theoretical analysis further supports the empirical findings on knowledge circuits in Section 5 of \citet{MoE-knowledge}, which show how experts in Qwen 1.5-MoE specialize in structured key-value associations like name-birthplace and country-capital pairs (see also \citet{MoE-knowledge-localization} for a related analysis with multilingual datasets). 



\section{Preliminaries} \label{preliminaries}

\subsection{Data model for knowledge}
\label{data}

We start with a simplified theoretical model for structured knowledge datasets, which are defined using a collection of syntactic patterns in conjunction with a finite set of subject–object associations (see \cite{factual-recall}, and Section 4 of \cite{nichani24}, for similar models of synthetic factual recall tasks). The patterns are represented by templates containing designated wildcard positions, which are instantiated by replacing the wildcards with words drawn from a task-specific knowledge table. This framework allows us to separate syntactic structure (captured by templates) from factual content (captured by key–value pairs).

\begin{definition}[Templates] \label{templates} Let \(\Sigma\) be a finite alphabet containing every token (word or punctuation mark) that may occur in the corpus. We introduce two distinguished wildcard symbols \(\star_{1},\star_{2}\notin\Sigma\). 
\leavevmode
A \emph{template} is a string
          \[ \tau \;=\; w_{1}w_{2}\dots w_{L}
             \;\in\;(\Sigma\cup\ \{ \star_{1},\star_{2}\})^{\ast}, \]
that contains each wildcard \(\star_{1}\) and \(\star_{2}\) exactly once, with \(\star_{1}\) occurring before \(\star_{2}\).

For any template \(\tau\) and any pair \((k,v)\in \Sigma \),
define the \emph{instantiation} $\tau[k,v]$ to be the string obtained by simultaneously replacing \(\star_{1}\mapsto k\) and \(\star_{2}\mapsto v\) in \(\tau\).
 \end{definition}
\begin{definition}[Knowledge tasks]
\label{knowledge-tasks}

A knowledge task $\underline{k}$ consists of a knowledge table $\Delta_{\underline{k}}$, and a set of templates $T(\underline{k})$. Here a knowledge table is a finite subset \( \Delta_{\underline{k}} \;\subseteq\;\Sigma^{\ast}\times\Sigma^{\ast}, \) whose first component is called the \emph{subject} and whose second component is the associated \emph{object}.

Define the sentences $\mathcal{D}(\underline{k})$ generated by the collection as follows.
\[ \mathcal{D}(\underline{k}) \;=\; \bigl\{\,
           \tau[k,v]\;:\;
           \tau \in T(\underline{k}),\;
           (k,v) \in \Delta_{\underline{k}} \bigr\} \]

Below we define the \emph{task complexity} $c(\underline{k})$ of a given task $\underline{k}$ as the combined size of its syntactic patterns and semantic facts. Here $|\tau|$ is the length of the template and $|\Delta_{\underline{k}}|$ is the number of key-value pairs in the knowledge table. $$ c(\underline{k}) = \sum_{\tau \in T(\underline{k})} |\tau| + |\Delta_{\underline{k}}| $$

\end{definition}

\begin{definition}[Knowledge dataset]
Define a knowledge dataset to be a pair $(\mathcal{K}, \mathcal{T})$, where $\mathcal{K}$ is a set of knowledge tasks and $\mathcal{T}$ is a set of templates. Here for each knowledge task $\underline{k} \in \mathcal{K}$, the set of templates $T(\underline{k})$ is a subset of $\mathcal{T}$. Define the \emph{sentence set} generated by the collection as follows. 
\[ D(\mathcal{K}, \mathcal{T})
      \;=\; 
      \bigcup_{\underline{k} \in \mathcal{K}}
      \mathcal{D}(\underline{k}) . \]
That is, we instantiate every template in every family
with every key–value pair drawn from the corresponding dictionary and then take the union over all knowledge tasks. \end{definition}

\paragraph{Example.}
Consider the two templates  
\(\tau_1 = \text{``}\star_{1}\text{ is in }\star_{2}\text{.''}\), \quad
\(\tau_2 = \text{``}\star_{1}\text{ is spoken widely in }\star_{2}\text{.''}\),  
with dictionaries  
\(\Delta_1=\{(\text{Paris},\text{France}),\,(\text{Madrid},\text{Spain})\}\), \quad
\(\Delta_2=\{(\text{English},\text{Canada}),\,(\text{Hindi},\text{India})\}\).  
Applying Definition~\ref{knowledge-tasks} gives the following, using the knowledge tasks $\underline{k_1} = (\Delta_1, \{ \tau_1 \})$ and $\underline{k_2} = (\Delta_2, \{ \tau_2 \} )$. 
\[ \begin{aligned} 
\mathcal{D}(\underline{k_1}) = \{ &\text{``Paris is in France.''},\; \\ 
&\text{``Madrid is in Spain.''} \}, 
\end{aligned} \]

\[ \begin{aligned} 
\mathcal{D}(\underline{k_2}) = \{ &\text{``English is spoken widely in Canada.''},\; \\ 
&\text{``Hindi is spoken widely in India.''} \}. 
\end{aligned} \]

\subsection{Transformer models} \label{transformer-model}

\paragraph{Dense transformers.}

We assume familiarity with transformers \cite{vaswani2017}; see Appendix \ref{def} for the definitions. A \emph{transformer block} first applies multi-head self-attention, then a position-wise feedforward neural network, wrapping each sub-block in residual connections. The multi-head attention mechanism ($\text{MHA}$) processes an input sequence by projecting it into queries, keys, and values using learned weight matrices $W_Q, W_K$, and $W_V$. For each attention head, the output is formed by computing attention weights via a softmax using scaled dot products of the queries and keys. The outputs from multiple heads are concatenated via an output matrix $W_O$. A \emph{transformer model} $\mathcal{M}$ is obtained by composing transformer blocks followed by a linear output head. 

We define a \emph{circuit} $C$ as a sparse sub-network within the transformer $\mathcal{M}$, consisting of a designated subset of attention heads and feedforward neurons \cite{knowledge-circuits}. We denote by $\mathcal{M}_C$ the restricted model where all components outside of $C$ are set to zero. 

\paragraph{Mixture of experts transformers.}

A mixture-of-experts (MoE) transformer modifies the feedforward component of a transformer layer by replacing a single feedforward network with a collection of expert feedforward networks and a gating mechanism that selects experts per position \cite{shazeer17}. For simplicity, we restrict our formulation to top-1 expert routing, activating only a single expert per token \cite{switch22}.

Let $\mathrm{FF}_1,\dots,\mathrm{FF}_E$ be $E$ position-wise feedforward networks, each with input/output width $d$. The gating function $G:\mathbb{R}^d \to \mathbb{R}^E,$ is defined by a linear map followed by a softmax.

The top-$1$ MoE feedforward layer is then given as follows, and is applied independently at each position. For an input vector $x\in\mathbb{R}^d$, here $e^\ast(x)$ denotes the index of the expert.   
\begin{align*}
e^\ast(x) \;=\; \arg\max_{e\in\{1,\dots,E\}} \bigl(G(x)\bigr)_e \\
\mathrm{MoE}(x) \;=\; \mathrm{FF}_{e^\ast(x)}(x),
\end{align*}

For an input sequence $A\in\mathbb{R}^{n\times d}$, the MoE transformer layer with $H$ attention heads is defined as \[ A'  = \mathrm{MHA}(A) + A, \qquad 
A'' = \mathrm{MoE}(A') + A'. \]
A mixture-of-experts transformer model is obtained by stacking such layers and adding a linear output head, as in the dense case.

\section{Theoretical results.} \label{theorems}

\subsection{Preliminaries: knowledge circuits in transformers.}

We first investigate how structured knowledge is represented within dense transformers. A core premise of our framework is that they have the capacity to partition learned facts into localized, decoupled sub-networks, commonly referred to as \emph{knowledge circuits}. The following theorem formalizes this intuition using our synthetic data model, and shows that a single-layer transformer can encode knowledge, using a sparse circuits to perform the computation for any individual task. Because the size of this task-specific circuit depends only on the task's intrinsic complexity, this result provides the fundamental theoretical justification for task routing in MoE architectures, which can use a router to eliminate redundant computation. 

\begin{theorem}[Sparse circuits in a single–layer transformer]
\label{sparse-circuits}
Let $\Sigma$ be a finite alphabet, let  
$\mathcal{T} \subseteq(\Sigma\cup\{\star_{1},\star_{2}\})^{\ast}$ be any finite set of wildcard templates, each with length at most $L$. Let $\mathcal{K}$ be a finite set of knowledge tasks. For each task $\underline{k} \in \mathcal{K}$, let $\Delta_{\underline{k}} \subseteq\Sigma^{\ast}\times\Sigma^{\ast}$ 
be the corresponding dictionary with size $|\Delta_{\underline{k}}| < N$, and let $T(\underline{k}) \subseteq \mathcal{T}$ be the corresponding set of templates.

For any $\epsilon > 0$, there exists a transformer $\mathcal{M}$ with the following properties.

\begin{itemize}
\item The model $\mathcal{M}$ has one transformer layer with dimension $2L|\mathcal{T}|$. It consists of a multi-head attention layer with $2$ attention heads, followed by a feedforward network with at most $L|\mathcal{T}| + N|\mathcal{K}|$ neurons. This transformer block is then followed by a linear output layer.

\item For each task $\underline{k} \in \mathcal{K}$, there exists a \emph{sparse circuit} $C(\underline{k})$ with at most $L|T(\underline{k})|+N$ neurons in the hidden layer of the feedforward network that solves this task.
\[ \mathcal{M}_{C(\underline{k})}(x) = \mathcal{M}(x) \qquad\text{for every } x\in \mathcal{D}(\underline{k}) \]

\item The model can achieve an error of less than \(\epsilon\): for any prefix \(\underline{w}\), the model's output probability distribution \( \mathcal{M}(\underline{w}) = P \) satisfies the following. Here \(V(\underline{w}) \subseteq \Sigma\) denotes the set of all valid next tokens from the distribution \(D(\mathcal{K}, \mathcal{T})\). $$ \sum_{w \in V(\underline{w})} P(w) > 1 - \epsilon $$

\end{itemize}

\end{theorem}

This theoretical construction provides a formal mathematical foundation for the empirical phenomena recently observed by \cite{knowledge-circuits}, which demonstrate that large, pretrained language models naturally encode specific relational facts within highly localized, sparse sub-networks, which they term ``knowledge circuits.'' While their findings rely on empirical techniques like causal tracing to locate these circuits, Theorem \ref{sparse-circuits} proves that transformers have the structural capacity to modularize knowledge using sparse knowledge circuits. 

\subsection{Main results: task routing in MoE transformers.}

Recent empirical studies in mechanistic interpretability have demonstrated that MoE transformers naturally develop specialized experts that handle distinct knowledge tasks \cite{MoE-knowledge}. These findings lack a formal guarantee that such MoE models can perfectly isolate discrete factual associations, and the following theorem bridges this gap by providing a mathematically rigorous foundation for these empirical observations. We establish that a single-layer MoE transformer can explicitly partition symbolic knowledge into distinct, task-specific experts without entanglement. Crucially, our construction proves that the required capacity of each expert scales strictly with the complexity of its assigned templates and dictionary. 

\begin{theorem}[Task routing in a single–layer mixture-of-experts transformer] \label{MoE}
Let $\Sigma$ be a finite alphabet, let  
$\mathcal{T} \subseteq(\Sigma\cup\{\star_{1},\star_{2}\})^{\ast}$ be any finite set of wildcard templates, each with length at most $L$. Let $\mathcal{K}$ be a finite set of knowledge tasks. For each task $\underline{k} \in \mathcal{K}$, let $\Delta_{\underline{k}} \subseteq\Sigma^{\ast}\times\Sigma^{\ast}$ 
be the corresponding dictionary with size $|\Delta_{\underline{k}}| < N$, and let $T(\underline{k}) \subseteq \mathcal{T}$ be the corresponding set of templates.

For any $\epsilon > 0$, there exists a transformer $\mathcal{M}$ with the following properties.

\begin{itemize}

\item The model $\mathcal{M}$ has one transformer layer with dimension $2L|\mathcal{T}|$. It consists of a multi-head attention layer with $2$ attention heads, followed by a mixture-of-experts layer with experts indexed by $\mathcal{K}$. The expert corresponding to $\underline{k} \in \mathcal{K}$ has at most $L|T(\underline{k})|+N$ neurons. This transformer layer is then followed by a linear output layer.

\item Given an input sequence $\underline{w}$ consisting of the first $i$ words generated by a template $\tau \in T(\underline{k})$, the transformer routes the sequence to the expert corresponding to $\underline{k}$.

\item The model can achieve an error of less than \(\epsilon\) in predicting the valid next words. Specifically, for any prefix \(\underline{w}\), the model's output probability distribution \(P\) satisfies the following. Here \(V(\underline{w}) \subseteq \Sigma\) denotes the set of all valid next tokens from the distribution \(D(\mathcal{K}, \mathcal{T})\). $$ \sum_{w \in V(\underline{w})} P(w \mid \underline{w}) > 1 - \epsilon $$ \end{itemize}

\end{theorem}

While recent theoretical work has analyzed the dynamics of optimization in MoE models \cite{MoE-theory, MoE-continual}, our result focuses on expressiveness and diverges from these frameworks in two critical aspects. First, these studies analyze isolated MoE layers, consisting solely of a router and feedforward experts, without using attention layers. Second, they define tasks using clustering approaches, such as Gaussian mixtures, which do not effectively represent natural language processing tasks. In contrast, Theorem \ref{MoE} explicitly integrates the multi-head attention mechanism with the MoE layer to process data, using a symbolic data model built on structural templates and relational dictionaries.


\section{Proofs of key results} \label{proofs}

In this section, we analyze the expressiveness of MoE transformer models and establish a rigorous mathematical foundation for explicit task-expert specialization. We prove that while dense transformers have the structural capacity to encode facts using sparse knowledge circuits, mixture-of-experts architectures can explicitly modularize this computation by routing discrete tasks to specialized experts whose capacity scales strictly with task complexity.

\subsection{Proof of Theorem \ref{sparse-circuits}.}

In this section, we outline the proof of Theorem \ref{sparse-circuits}. Given an input sequence $\underline{w} = (w_j)_{1 \leq j \leq i}$, and an index $c \in \{ 1, 2 \}$, we denote by $a_{c}(\underline{w}) \in \mathbb{R}$ the image of $\underline{w}$ in the corresponding attention head. We denote by $z'(\underline{w}) \in \mathbb{R}^{2L|\mathcal{T}|}$ the image after concatenating and passing through the $W_O$ matrix, and $z(\underline{w}) \in \mathbb{R}^{2L|\mathcal{T}|}$ the image after passing through the feedforward network. 

\begin{proof}

\textbf{Construction of embedding layer.} We partition the embedding dimension into two blocks: one dedicated to encoding the tokens that appear as keys or values in the dictionaries, and the other for the syntactic tokens forming the templates. 

\textbf{Construction of attention heads.} Let the input sequence $\underline{w}$ consist of the first $i$ words generated by a template $\tau \in T(\underline{k})$ and a subject-object pair $(k, v) \in \Delta_{\underline{k}}$, for a given task $\underline{k} \in \mathcal{K}$. We choose the query, key and value matrices of the first attention head so that $a_1(\underline{w}) = v_i(\tau) \in \mathbb{R}^{L|\mathcal{T}|}$ is a vector that depends on the template $\tau$ and the index $i$, and not the subject $k$. We also choose the weights so that the vectors $\{ v_i(\tau) \}$ are linearly independent in $\mathbb{R}^{2L|\mathcal{T}|}$. We choose the query, key and value matrices of the second attention head so that $a_2(\underline{w})$ is a vector that depends only on the subject $k$. We also set $W_O$ to be the identity matrix, so that $z'(\underline{w})$ is obtained by concatenating $a_1(\underline{w})$ and $a_2(\underline{w})$. 

\textbf{Construction of feedforward network.} 

We decompose the neurons $I$ in the hidden layer into $2|\mathcal{K}|$ blocks as follows.
$$ I = \bigsqcup_{\underline{k} \in \mathcal{K}} I_{\underline{k}} \cup \bigsqcup_{\underline{k} \in \mathcal{K}} I'_{\underline{k}} $$

For each $\underline{k} \in \mathcal{K}$, the sparse circuit $C(\underline{k})$ is obtained by selecting neurons in $I_{\underline{k}}$, which are used for syntactic tokens, and $I'_{\underline{k}}$ which are used for factual recall. To enforce strict circuit isolation, large negative biases are applied to all off-task neurons, ensuring that the ReLU activations are zero for any task $\underline{k}' \neq \underline{k}$. The syntactic block $I_{\underline{k}}$ connects exclusively to the output of the first attention head, while the semantic block $I'_{\underline{k}}$ connects exclusively to the output of the second attention head.

We use Lemma \ref{memorize} to construct the weights of these blocks, which memorize the respective mappings. The syntactic block has $|I_{\underline{k}}| = L|T(\underline{k})|$ neurons to map all unique template prefixes to vectors that encode next-word tokens, while the semantic block utilizes $|I'_{\underline{k}}| \leq N$ neurons to map keys to vectors that encode their corresponding values. 

\textbf{Construction of the output layer.} 

\begin{lemma} \label{unembedding}
For every integer $n \ge 3$, there exist vectors $v_1,\dots,v_n \in \mathbb{R}^n$ spanning a subspace of dimension $2$ such that, for each $i \in \{1,\dots,n\}$, the $i$-th coordinate of $v_i$ is the unique maximal coordinate of $v_i$. \end{lemma}

To construct the output logits, we use the above Lemma, which implies that for each token in the vocabulary, we can construct a vector within a 2-dimensional subspace of the output logits, such that the maximal coordinate of that vector is indexed by the corresponding token. We construct the matrix $W_U$ so that its output is this 2-dimensional subspace, and choose the target outputs of the feedforward network above to be the pre-image of these vectors under the mapping induced by $W_U$. By applying a sufficiently large scaling factor to the $W_U$, we can ensure that the resulting softmax probabilities for the target token are larger than $> 1 - \epsilon$ for a given $\epsilon$, concluding the proof. \end{proof}

\begin{figure*}[!ht]
    \centering
    \begin{subfigure}[t]{0.32\linewidth}
        \centering
        \includegraphics[width=\linewidth]{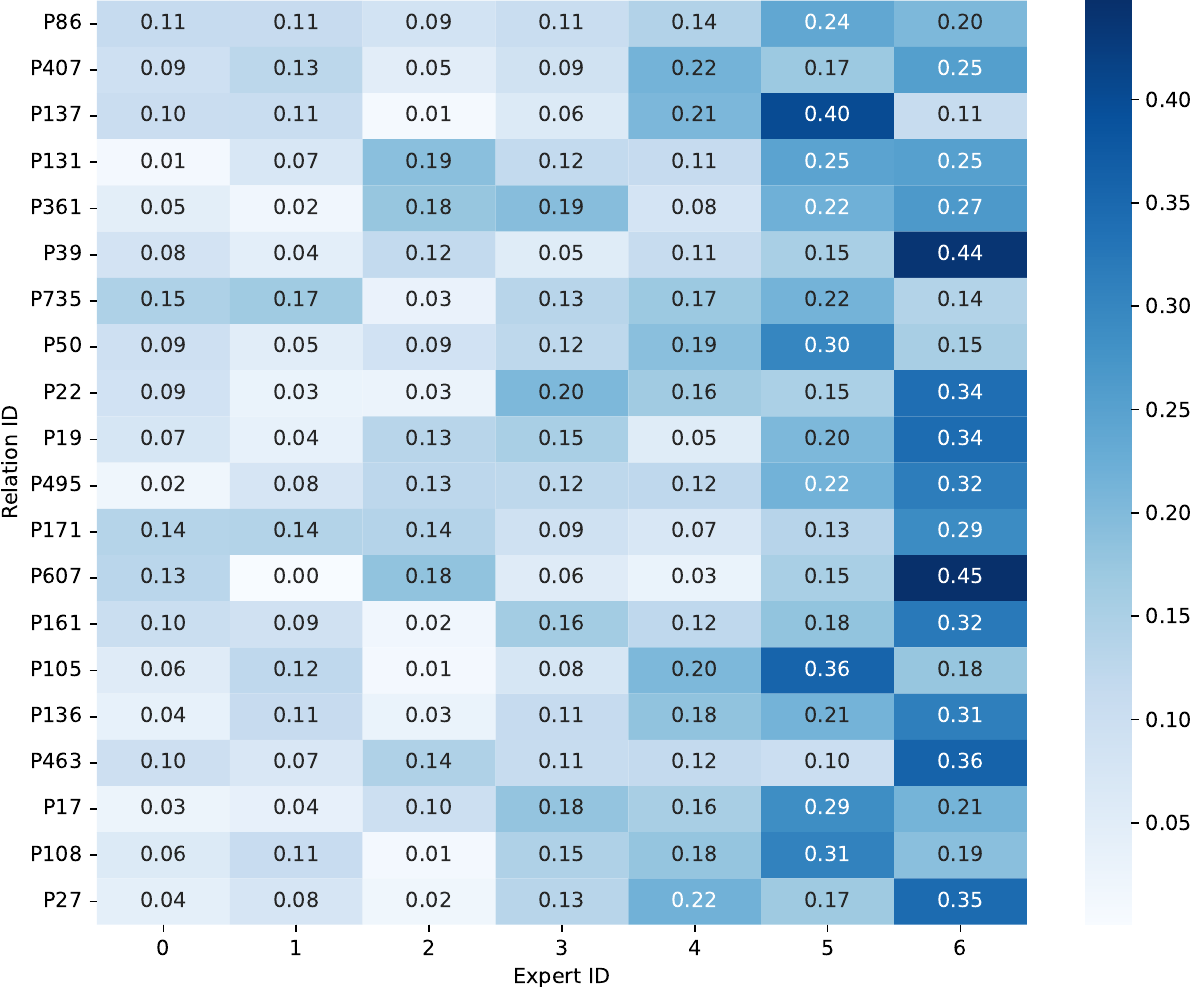}
        \caption{CE loss only}
        \label{fig:ce_heatmap}
    \end{subfigure}
    \hfill
    \begin{subfigure}[t]{0.32\linewidth}
        \centering
        \includegraphics[width=\linewidth]{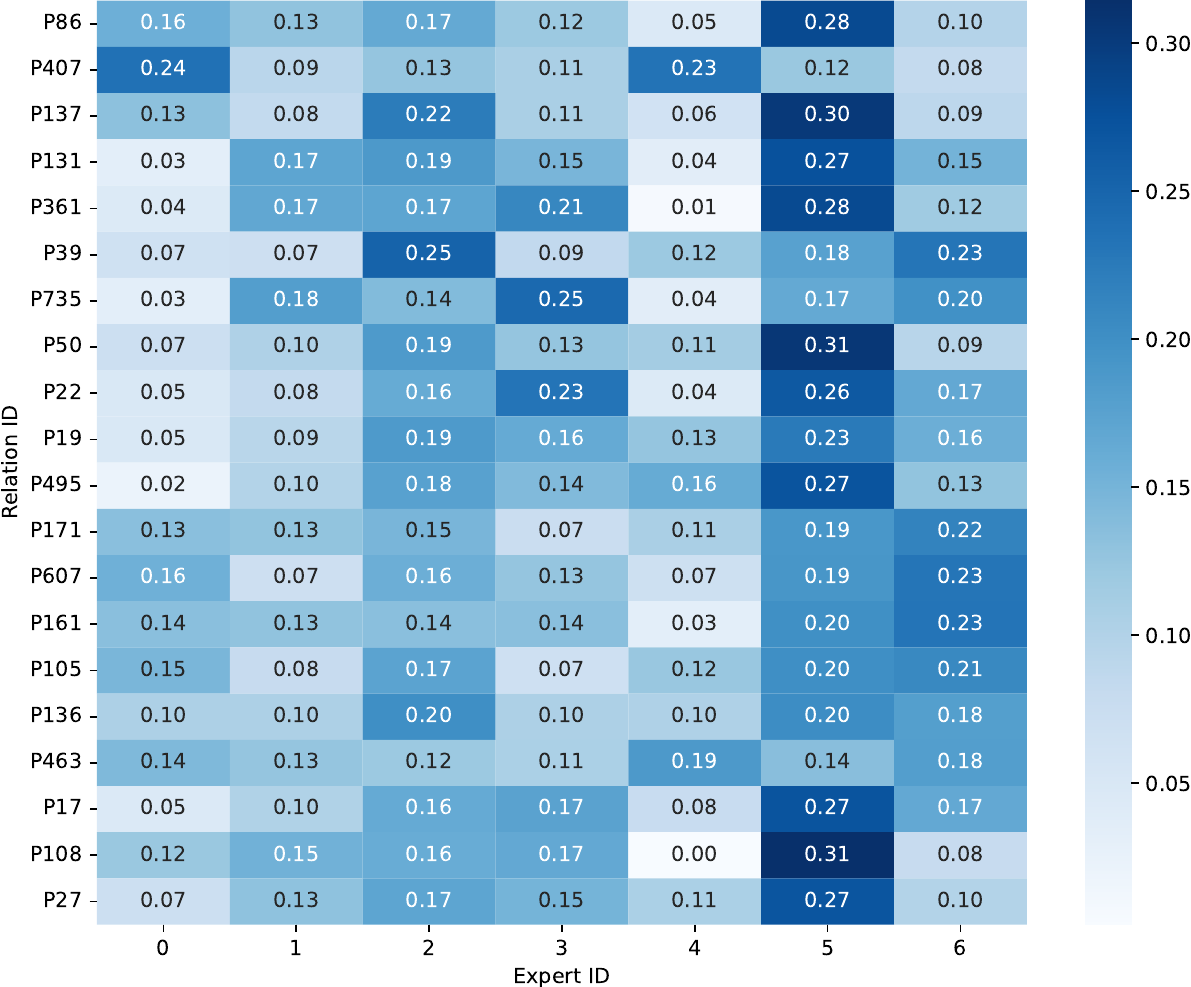}
        \caption{CE + Load Balancing loss}
        \label{fig:lb_heatmap}
    \end{subfigure}
    \hfill
    \begin{subfigure}[t]{0.32\linewidth}
        \centering        \includegraphics[width=\linewidth]{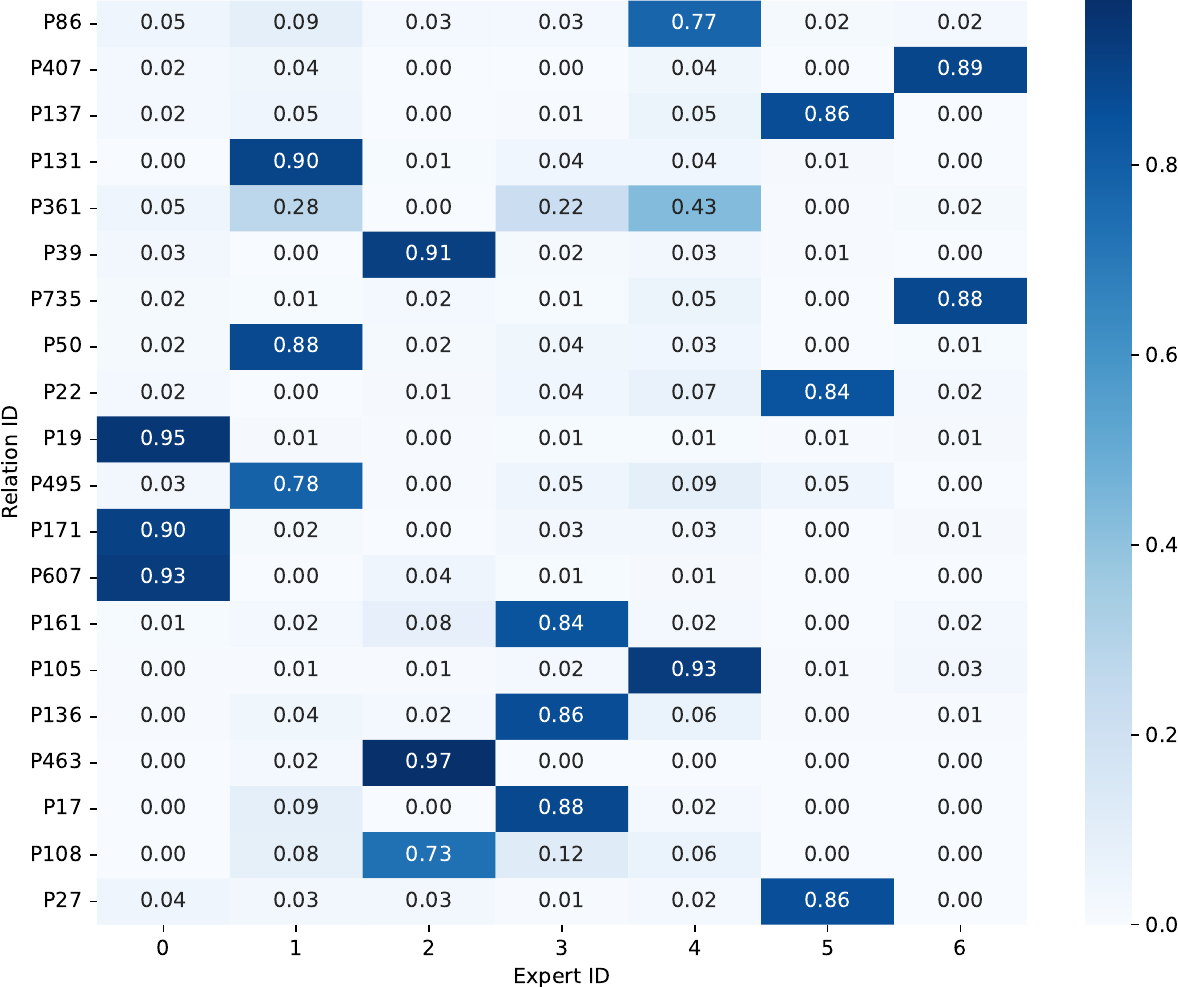}
        \caption{CE + Router loss}
        \label{fig:router_heatmap}
    \end{subfigure}
    \caption{Task-Expert routing distributions under different training objectives with 20 tasks and 7 experts.}
    \label{fig:heatmap}
    \vspace{-4mm}
\end{figure*}

\subsection{Proof of Theorem \ref{MoE}.}

In this section, we outline the proof of Theorem \ref{MoE}. Let $\underline{w} = (w_j)_{1 \leq j \leq i}$ be an input sequence corresponding to a task $\underline{k} \in \mathcal{K}$, that consists of the first $i$ words generated by a template $\tau \in T(\underline{k})$ and a subject-object pair $(k, v) \in \Delta_{\underline{k}}$. Given an index $c \in \{ 1, 2 \}$, we denote by $a_{c}(\underline{w}) \in \mathbb{R}$ the image of $\underline{w}$ in the corresponding attention head. We denote by $z'(\underline{w}) \in \mathbb{R}^{2L|\mathcal{T}|}$ the image after concatenating and passing through the $W_O$ matrix, and $z(\underline{w}) \in \mathbb{R}^{2L|\mathcal{T}|}$ the image after passing through the feedforward network. We partition the vectors $z(\underline{w})$ and $z'(\underline{w})$ into two block components, each of dimension $L|\mathcal{T}|$. 

\begin{proof}

\textbf{Construction of attention heads and embedding layer.} We follow the approach used in the proof of Theorem \ref{sparse-circuits} to construct the embedding layer, and choose the query, key and value matrices so that $a_1(\underline{w}) = v_i(\tau) \in \mathbb{R}^{L|\mathcal{T}|}$ is a vector that depends only on the template $\tau$ and the index $i$ (and not the subject $k$), while $a_2(\underline{w})$ depends only on the subject $k$. Our choice of these matrices also ensures that the vectors $\{ v_i(\tau) \}$ are linearly independent in $\mathbb{R}^{L|\mathcal{T}|}$.

\textbf{Construction of the router.} We construct a router $G$ with the property that for the input sequence $\underline{w}$, the chosen expert corresponds to the task $\underline{k}$. To do this, we choose the gating function $G: \mathbb{R}^{2L|\mathcal{T}|} \rightarrow \mathbb{R}^{|\mathcal{K}|}$ so that it only depends on the first $L|\mathcal{T}|$ coordinates of the input vector. Since the vectors $\{ v_i(\tau) \}$ are linearly independent, we can choose the linear map so that for each $i$ and template $\tau \in T(\underline{k})$, $G(v_i(\tau))$ is a one-hot vector in $\mathbb{R}^{|\mathcal{K}|}$ indexed by $\underline{k}$. 

\textbf{Construction of the experts.} To construct the expert for each task $\underline{k} \in \mathcal{K}$, we separate its neurons into two blocks, and follow the approach used in the proof of Theorem \ref{sparse-circuits}. The first block is responsible for identifying the next word in the template $\tau$, and maps the image of current template $v_i(\tau) \in \mathbb{R}^{L|\mathcal{T}|}$ to an embedding of the next word. The second block functions as an associative memory responsible for retrieving factual associations from the knowledge table $\Delta_{\underline{k}}$ using the representation $a_2(\underline{w})$ of the key $k$ for every dictionary entry $(k, v) \in \Delta_{\underline{k}}$. 

\textbf{Output logits and probability bound.} Finally, we map the expert's output $z(\underline{w})$ to a probability distribution over the vocabulary $\Sigma$ using a linear unembedding matrix $W_U$, following the proof of Theorem \ref{sparse-circuits}. When using Lemma \ref{memorize} above, for each of the two blocks we choose the outputs so that they lie in a $2$-dimensional subspace of $\mathbb{R}^{L|\mathcal{T}|}$.  We then construct a linear mapping $W_U$ between these two-dimensional subspaces, so that the logits $l = W_U z'(\underline{w})$ satisfy the following property: the logit for the valid next token $w \in V(\underline{w})$ is strictly greater than the logit for any invalid token $y \notin V(\underline{w})$. To satisfy the theorem's error bound, we apply a sufficiently large scaling factor $\gamma > 0$ to the weight matrix $W_U$, so that the model predicts the correct next tokens with a total probability greater than $1 - \epsilon$. \end{proof} 

\section{Experiments.} \label{experiments}

\begin{figure*}
    \centering    \includegraphics[width=1.0\linewidth]{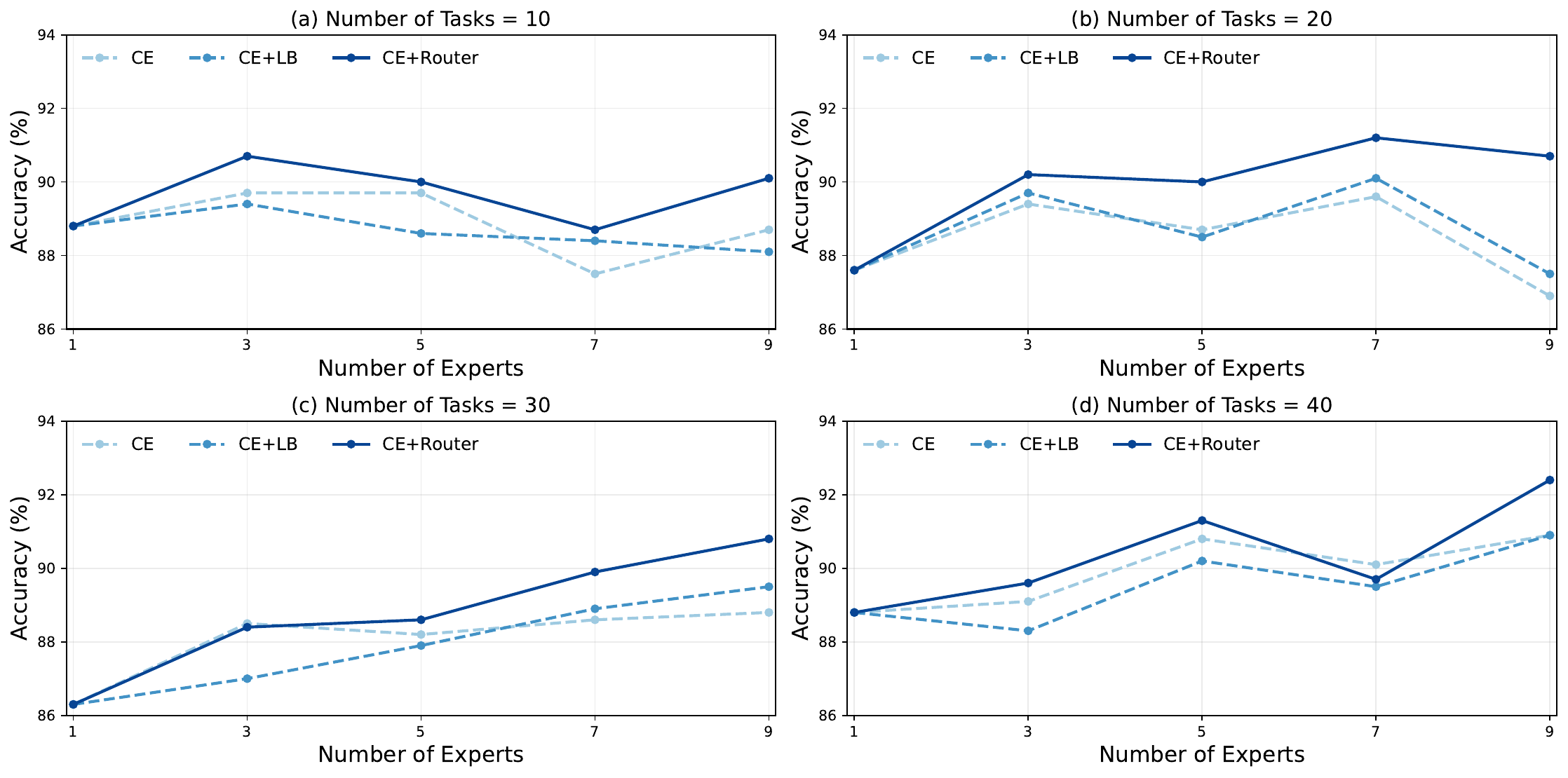}
    \caption{Accuracy under different training objectives across varying numbers of tasks and experts. CE$+$Router achieves comparable performance in most settings.}
    \label{fig:performance_plot}
    \vspace{-4mm}
\end{figure*}

\subsection{Dataset}

We construct a synthetic knowledge dataset from the setup in Section \ref{data}, using entity pairs and natural language templates from WikiData5M \cite{wikidata5M}, which is publicly available under an open-source license. Our use of WikiData5M is consistent with its intended use for academic research, and the synthetic dataset is similarly intended solely for research purposes. We sample 50 knowledge tasks from WikiData5M, where each task $\underline{k}$ is indexed by a relation IDs. The templates $T(\underline{k})$ are generated by sampling $5$ sentences from the relation ID using Gemini-2.5. The corresponding knowledge table $D(\underline{k})$ is obtained by sampling $200$ entity pairs for that relation. We tokenize each sentence using a custom tokenizer that preserves quoted entity names as atomic tokens. The dataset is then split into training and test subsets, with an $80\%-20\%$ train-test split. 

\subsection{Models} \label{moe-models}

We use a one-layer mixture-of-experts transformer model, following the setup in Section \ref{transformer-model}. We train the model on the synthetic dataset obtained from WikiData5M using stochastic gradient descent. We evaluate the model by autoregressively generating sequences from randomly truncated sentences from the test set, until the end-of-sentence token is produced or a maximum length is reached. The model's accuracy is the proportion of generated sequences that exactly match sentences in the dataset. We refer the reader to Appendix \ref{Appendix-B} for more details about the experimental setup. 

We optimize expert utilization using two distinct auxiliary objectives. The load balancing loss \cite{shazeer17, switch22} encourages an even distribution of tokens across all available experts to prevent routing bottlenecks. Our task router loss explicitly enforces task-specific routing by training the model to route each token directly to the expert associated with that task.

For a batch of $N$ tokens and $E$ experts, let $p_{x,e}$ and $z_{x,e}$ denote the routing probability and one-hot assignment, respectively, for token $x \in \{1, \dots, N\}$ at expert $e \in \{1, \dots, E\}$. Let $e_x \in \{1, \dots, E\}$ be the target task expert index for token $x$. The loss functions are defined as follows:
\[ \mathcal{L}_{\mathrm{LB}} = E \cdot \sum_{e=1}^E \left( \frac{1}{N} \sum_{x=1}^N p_{x,e} \right) \left( \frac{1}{N} \sum_{x=1}^N z_{x,e} \right) \]
\[ \mathcal{L}_{\text{router}} = - \frac{1}{N} \sum_{x=1}^{N} \log p_{x, e_x} \]
The total training objective is $\mathcal{L} = \mathcal{L}_{\text{CE}} + \lambda \mathcal{L}_{\text{aux}}$, where $\mathcal{L}_{\text{CE}}$ is the cross-entropy loss, $\mathcal{L}_{\text{aux}} \in \{ \mathcal{L}_{\mathrm{LB}}, \mathcal{L}_{\text{router}} \}$ is the chosen auxiliary objective, and $\lambda$ is a weighting coefficient (e.g., $\lambda = 0.2$).

\subsection{Routing specialization in trained mixture-of-experts models}


We examine the relationship between tasks and experts by visualizing token-level routing distributions across different objectives in Figure~\ref{fig:heatmap}. While training with cross-entropy alone (Figure~\ref{fig:ce_heatmap}) or with an additional load balancing loss (Figure~\ref{fig:lb_heatmap}) results in diffuse routing patterns that lack clear task--expert alignment, the task routing loss yields more uniform expert utilization. In contrast, incorporating task-based routing (Figure~\ref{fig:router_heatmap}) leads to substantially clearer correspondence, where each task is predominantly routed to a specific expert. 

\subsection{Model Performance}


We compare the overall model performance under the three different training objectives in Figure~\ref{fig:performance_plot}. As we vary the number of tasks and experts, the three variants achieve comparable accuracy, indicating that enforcing task-specific routing does not degrade predictive performance. Notably, the model trained with task-based routing typically achieves a slightly higher accuracy when more experts are present. These results suggest that explicit task-specific routing can also reduce the active parameter count while maintaining model performance (see Appendix~\ref{Appendix-B} for more details). 

\section{Conclusion.}

In this work, we establish a rigorous theoretical foundation for task-expert specialization in MoE transformers. By using a synthetic framework based on templates, we prove that a single-layer MoE transformer can solve this language modelling task using task-specific experts whose size scales with task complexity. Our experiments show that, in simplified settings with synthetic data, task-level routing encourages specialized experts to store knowledge. Our framework provides a mathematical blueprint for designing interpretable mixture-of-expert architectures with task-specialized experts, that could yield inference-time speedups by reducing the active parameter count. 

\section{Limitations.}

\subsection{Theoretical models for deep mixture-of-experts transformers.}

Our construction could be generalized to deeper mixture-of-experts (MoE) transformers by replacing templates with sentences generated using context-free grammars that can model the hierarchical structure of language \cite{zhao23}. In this setting, sentences are generated by derivation trees rather than single-step wildcard instantiations, with non-terminals corresponding to phrases in the sentence \cite{chomsky}. In a deep MoE transformer, earlier attention layers can be used to parse grammar, and later expert layers route tokens to experts that encode knowledge \cite{MoE-knowledge}. Extending this hierarchical formulation could provide a theoretical mechanism for modeling multi-step reasoning in transformer models, where deeper MoE layers sequentially route intermediate tokens through specialized experts to obtain a reasoning chain \cite{merrill2024}. 

\subsection{Sequence-level routing in mixture-of-experts models.}

Our theoretical framework supports the thesis in \cite{taskMoE}, which shows that routing entire inputs in mixture-of-expert models based on predefined task labels explicitly isolates domain knowledge into specialized sub-networks. While this approach relies on task labels for input data, in \cite{MoE-design} the authors show that sequence-level routing in a self-supervised setting is a viable alternative. It would be interesting to empirically compare our simplified experiments from Section \ref{moe-models} using token-level task routing in mixture-of-experts with the sequence-level routing approach from \cite{MoE-design}. In this paradigm, the gating network computes a single, global routing decision for the sequence, typically by pooling token representations, thereby forcing all tokens within that sequence to be processed by the same chosen expert without requiring task labels. In \cite{MoE-design}, it is observed that while token-level routing often causes experts to specialize in shallow syntactic features, sequence-level routing encourages experts to capture higher-level, domain-specific semantic concepts.

\section*{Acknowledgments}
We would like to thank Goncalo Paulo and Catherine Arnett for insightful discussions on transformer circuits. We utilized an AI assistant, Gemini 3 Pro, to iteratively draft and refine prose throughout the manuscript. We carefully reviewed and edited all generated text.

\vspace{0.5em}


\bibliography{references}

\appendix

\section{Appendix: Proofs}

\subsection{Definitions} \label{def}

\begin{definition} A position-wise feedforward network (FFN) with input/output width \(d\) and hidden width \(d_{\mathrm{ff}}\) is the function
\[
\mathrm{FF}:\mathbb{R}^d\to\mathbb{R}^d,\;
x\mapsto W^{(2)}\sigma(W^{(1)}x+b^{(1)})+b^{(2)},
\]
with \(W^{(1)}\in\mathbb{R}^{d_{\mathrm{ff}}\times d}\), \(b^{(1)}\in\mathbb{R}^{d_{\mathrm{ff}}}\), \(W^{(2)}\in\mathbb{R}^{d\times d_{\mathrm{ff}}}\), \(b^{(2)}\in\mathbb{R}^d\), and \(\sigma\) is the ReLU applied element-wise.
\end{definition}

\begin{definition} A transformer layer with $H$ heads and width $d$ is a function defined as follows.
\[ \text{Layer}: \mathbb{R}^{n \times d} \to \mathbb{R}^{n \times d}, \text{Layer}(A) = A'' \]
\[ A' = \text{MHA}(A) + A \]
\[ A'' = \text{FF}(A') + A' \]
Here $\text{MHA}$ is the multi-head attention mechanism, and $\text{FF}$ is a position-wise feedforward network (FFN). The terms $+A$ and $+A'$ are residual connections. $\blacksquare$
\end{definition} 

\subsection{Lemmas}

We start with a proof of Lemma \ref{unembedding} above.

\begin{lemma}
For every integer $n \ge 3$, there exist vectors $v_1 , \dots, v_n \in \mathbb{R}^n$ spanning a subspace of dimension $2$ such that, for each $i \in \{1,\dots,n\}$, the $i$-th coordinate of $v_i$ is the unique maximal coordinate of $v_i$.
\end{lemma}

\begin{proof} Let $\theta_k = \tfrac{2\pi(k-1)}{n}$ for $k=1,\dots,n$, and define the vectors below. 
\begin{align*}
u &= (\cos\theta_1, \cos\theta_2, \dots, \cos\theta_n)^\top, \\
w &= (\sin\theta_1, \sin\theta_2, \dots, \sin\theta_n)^\top. \\
v_k &= \cos\theta_k\, u + \sin\theta_k \\
(v_k)_j &= \cos\theta_k\cos\theta_j + \sin\theta_k\sin\theta_j \\
&= \cos(\theta_k-\theta_j).
\end{align*}
Hence $(v_k)_k = \cos 0 = 1$, while for $j\neq k$ we have $\cos(\theta_k-\theta_j) < 1$, so the maximum coordinate of $v_k$ is attained uniquely at position $k$. All $v_k$ lie in $\mathrm{span}\{u,w\}$, which has dimension $2$ because $u$ and $w$ are linearly independent when $n\ge 3$. This gives the desired family of $n$ vectors in a $2$-dimensional subspace of $\mathbb{R}^n$.
\end{proof}

We also recall a result from \cite{memorize2} that analyzes the memorization capacity of feedforward neural networks and shows that they can fit an arbitrary finite set of input-output vectors (see also \cite{memorize}).

\begin{lemma} \label{memorize}
Let $\{(x_i, y_i)\}_{i=1}^K$ be a finite dataset where each input $x_i \in \mathbb{R}^n$ and each output $y_i \in \mathbb{R}^q$.  
Assume that all inputs are distinct, i.e., $x_i \neq x_j$ for $i \neq j$.  
Consider a feedforward neural network $f_\theta$ with one hidden layer of width $K$, ReLU activations, input dimension $n$, and output dimension $q$.  
Then there exists a choice of parameters $\theta$ such that
\[ f_\theta(x_i) \;=\; y_i \quad \text{for all } i \in [K].
\] \end{lemma}

\subsection{Proof of Theorem \ref{sparse-circuits}.}

\begin{proof}

\textbf{Construction of attention heads.} Recall that the embedding dimension is partitioned into two components: the first for template tokens and the second for dictionary tokens. We choose the query and key matrices of the first attention head so that they operate exclusively on coordinates within the first component. This yields an output $a_1(\underline{w}) = v_i(\tau) \in \mathbb{R}^{L|\mathcal{T}|}$ that depends exclusively on the underlying template $\tau \in \mathcal{T}$ and the current sequence index $i$. Similarly, we choose the query and key matrices of the second attention head so that they operate exclusively on coordinates within the second component. The output vector $a_2(\underline{w})$ depends only on the dictionary tokens present in the input sequence $\underline{w}$. 

Because the total number of unique template prefixes is bounded by $L|\mathcal{T}|$, which does not exceed the dimension of the target space $\mathbb{R}^{L|\mathcal{T}|}$, we can rely on a generic choice of the matrices $W_Q, W_K$ and $W_V$ to ensure the linear independence of the vectors $\{ v_i(\tau) \}$. Using the standard attention mechanism, each output vector $v_i(\tau)$ can be expressed as a linear combination below, where $\alpha_j$ are the attention weights, and $x_j$ denotes the input embedding (comprising both the token embedding and its positional encoding) for the $j$-th token in the sequence. $$v_i(\tau) = \sum_{j=1}^i \alpha_j (x_j W_V)$$ Because the query and key matrices only operate on the first component, the attention weights $\alpha_j$ are uniquely determined by the first $i$ tokens in the template $\tau$. With the standard positional encodings, each of these prefixes yields a distinct combination of token representations and attention weights. Since there are at most $L|\mathcal{T}|$ distinct vectors $\{ v_i(\tau) \}$ in a space of dimension $L|\mathcal{T}|$, for a generic choice of weight matrices $W_Q, W_K$ and $W_V$, the vectors will be linearly independent.

\textbf{Construction of feedforward network.} We denote by $h(\underline{w})$ the activations of the hidden layer of the feedforward network. We decompose the neurons $I$ in the hidden layer into $2|\mathcal{K}|$ blocks as follows.
$$ I = \bigsqcup_{\underline{k} \in \mathcal{K}} I_{\underline{k}} \cup \bigsqcup_{\underline{k} \in \mathcal{K}} I'_{\underline{k}} $$
Here $|I_{\underline{k}}| = L|T(\underline{k})|$ and $|I'_{\underline{k}}| \le N$ neurons. We choose the weights of the feedforward network such that the following properties hold.
\begin{itemize} 
\item \textbf{Sparse activation (Circuit isolation):} By adding a large negative bias to neurons not associated with the current task, we ensure that if $\underline{k}' \neq \underline{k}$, then the ReLU activation yields $h_j(\underline{w}) = 0$ for all neurons $j \in I_{\underline{k}'} \cup I'_{\underline{k}'}$. This isolates the computation to the sparse circuit $C(\underline{k}) = I_{\underline{k}} \cup I'_{\underline{k}}$.
\item \textbf{Decoupled processing:} The block $I_{\underline{k}}$ connects exclusively to the first half of $z'(\underline{w})$ (the template representation $a_1(\underline{w})$), while the block $I'_{\underline{k}}$ connects exclusively to the second half (the subject representation $a_2(\underline{w})$). The final output of the network, $z(\underline{w})$, consists of two blocks: the first predicts the next structural word of the template (computed solely by $I_{\underline{k}}$), and the second predicts the factual object from the dictionary (computed solely by $I'_{\underline{k}}$). \end{itemize}

We use Lemma \ref{memorize} above to construct the weights for these blocks. The template block $I_{\underline{k}}$, maps the representation $a_1(\underline{w})$ to a vector representation of the next syntactic token. Because there are at most $L|T(\underline{k})|$ unique linearly independent prefixes, Lemma \ref{memorize} guarantees this exact mapping can be memorized using $|I_{\underline{k}}|$ neurons. Concurrently, the dictionary block $I'_{\underline{k}}$ must map the subject representations $a_2(\underline{w})$ to their corresponding object representations. Since there are at most $N$ entries in $\Delta_{\underline{k}}$, Lemma \ref{memorize} guarantees this can be memorized using $|I'_{\underline{k}}| \le N$ neurons. 

\textbf{Construction of the output layer.} 

Finally, we map the output $z(\underline{w})$ of the transformer block to a probability distribution over the vocabulary $\Sigma$.  To construct the weights $W_U$ of the linear unembedding layer, which maps $z(\underline{w})$ to the logits $l$, we use Lemma \ref{unembedding}. Letting $n = |\Sigma|$ denote the vocabulary size, by Lemma \ref{unembedding}, there exist $n$ vectors $\{v_w\}_{w \in \Sigma} \subset \mathbb{R}^n$ spanning a $2$-dimensional subspace, such that the $w$-th coordinate of $v_w$ is its unique maximal coordinate. If $\{b_1, b_2\}$ is a basis for this subspace, each target vector $v_w$ corresponding to a valid token $w$ can be uniquely expressed as a linear combination of these basis vectors: $v_w = \alpha_w b_1 + \beta_w b_2$ for some scalars $\alpha_w, \beta_w \in \mathbb{R}$.

When applying Lemma \ref{memorize} to construct the weights of the feedforward network, we choose the output vectors so that they lie within a $2$-dimensional subspace of $\mathbb{R}^{L|\mathcal{T}|}$. The linear unembedding matrix $W_U \in \mathbb{R}^{n \times 2L|\mathcal{T}|}$ is used to construct a mapping between these $2$-dimensional subspaces, so that that the output is a linear combination of $v_w$ for $w \in V(\underline{w})$. Because the $w$-th coordinate of $v_w$ is uniquely maximal by Lemma \ref{unembedding}, it guarantees that the logit $l_w$ for the valid next token is strictly greater than the logit $l_y$ for any invalid token $y \notin V(\underline{w})$.


To ensure that the model has error at most $\epsilon$, we scale the weight matrix $W_U$ by a sufficiently large constant $\gamma > 0$. Recalling that the final probability $P(w)$ for any token $w$ is given by the softmax function below. 
$$P(w) = \frac{\exp(\gamma l_w)}{\sum_{y \in \Sigma} \exp(\gamma l_y)}$$
Because the logits of the valid tokens are strictly greater than those of the invalid tokens, taking the limit as $\gamma \to \infty$ forces the probability mass to concentrate entirely on the set $V(\underline{w})$ of valid next-word tokens. Therefore, for any strictly positive $\epsilon > 0$, there exists a scaling factor $\gamma$ such that the probability distribution satisfies $\sum_{w \in V(\underline{w})} P(w) > 1 - \epsilon$. This concludes the proof of the theorem. \end{proof}

\begin{table*}[t!]
\centering
\small
\caption{Model performance under different MoE loss formulations across varying numbers of tasks and experts. We report accuracy (\%) averaged over three runs to ensure robustness.}
\label{tab:performance}
\begin{tabular}{|c|c|c|c|c|}
\hline
num\_tasks & num\_experts & acc (ce loss) & acc (ce + router loss) & acc (ce + load balancing loss) \\
\hline
\multirow{5}{*}{10}
 & 1 & 88.8\% & 88.8\% & 88.8\% \\ \cline{2-5}
 & 3 & 89.7\% & 90.7\% & 89.4\% \\ \cline{2-5}
 & 5 & 89.7\% & 90.0\% & 88.6\% \\ \cline{2-5}
 & 7 & 87.5\% & 88.7\% & 88.4\% \\ \cline{2-5}
 & 9 & 88.7\% & 90.1\% & 88.1\% \\
\hline
\multirow{5}{*}{20}
 & 1 & 87.6\% & 87.6\% & 87.6\% \\ \cline{2-5}
 & 3 & 89.4\% & 90.2\% & 89.7\% \\ \cline{2-5}
 & 5 & 88.7\% & 90.0\% & 88.5\% \\ \cline{2-5}
 & 7 & 89.6\% & 91.2\% & 90.1\% \\ \cline{2-5}
 & 9 & 86.9\% & 90.7\% & 87.5\% \\
\hline
\multirow{5}{*}{30}
 & 1 & 86.3\% & 86.3\% & 86.3\% \\ \cline{2-5}
 & 3 & 88.5\% & 88.4\% & 87.0\% \\ \cline{2-5}
 & 5 & 88.2\% & 88.6\% & 87.9\% \\ \cline{2-5}
 & 7 & 88.6\% & 89.9\% & 88.9\% \\ \cline{2-5}
 & 9 & 88.8\% & 90.8\% & 89.5\% \\
\hline
\multirow{5}{*}{40}
 & 1 & 88.8\% & 88.8\% & 88.8\% \\ \cline{2-5}
 & 3 & 89.1\% & 89.6\% & 88.3\% \\ \cline{2-5}
 & 5 & 90.8\% & 91.3\% & 90.2\% \\ \cline{2-5}
 & 7 & 90.1\% & 89.7\% & 89.5\% \\ \cline{2-5}
 & 9 & 90.9\% & 92.4\% & 90.9\% \\
\hline
\end{tabular}
\end{table*}

\subsection{Proof of Theorem \ref{MoE}.}
\begin{proof}
\textbf{Construction of attention heads.} We follow the approach used in the proof of Theorem \ref{sparse-circuits}, and choose the query, key and value matrices so that $a_1(\underline{w}) = v_i(\tau) \in \mathbb{R}^{L|\mathcal{T}|}$ is a vector that depends only on the template $\tau$ and the index $i$ (and not the subject $k$), while $a_2(\underline{w})$ depends only on the subject $k$. Our choice of these matrices also ensures that the vectors $\{ v_i(\tau) \}$ are linearly independent in $\mathbb{R}^{L|\mathcal{T}|}$.

\textbf{Construction of the router.} We construct a linear router $G: \mathbb{R}^{2L|\mathcal{T}|} \rightarrow \mathbb{R}^{|\mathcal{K}|}$ whose input is the concatenated representation $z'(\underline{w}) = [a_1(\underline{w}); a_2(\underline{w})]$. We set the weights corresponding to the second component $a_2(\underline{w})$ to strictly zero, so that the routing decision depends solely on the syntactic template: $a_1(\underline{w}) = v_i(\tau)$. Since the vectors $\{ v_i(\tau) \}$ are linearly independent, we can choose the weight matrix of $G$ that maps each $v_i(\tau)$ (for $\tau \in T(\underline{k})$) exactly to the one-hot basis vector $e_{\underline{k}} \in \mathbb{R}^{|\mathcal{K}|}$. 
This ensures that all tokens generated by templates in $T(\underline{k})$ are routed to the $\underline{k}$-th expert. 

\textbf{Construction of the experts.} For each $\underline{k} \in \mathcal{K}$, we separate the neurons of the corresponding expert into two block components. We choose the weights such that the first and second blocks of the input $z(\underline{w})$ are routed strictly through their respective expert components and mapped to the corresponding blocks of the output $z'(\underline{w})$, setting all other weights to zero. We use Lemma \ref{memorize}, about memorizing input-output vectors with a neural network with a single hidden layer, to construct the weights. 

\begin{figure}[h!]
    \centering
    \caption{Example of template-based generation from WikiData5M. Given templates $\mathcal{T}$ and knowledge pairs $\mathcal{D}$, sentences are generated by filling subject-object pairs into templates.}
    \label{fig:example}
    \vspace{-4mm}
    \begin{promptbox}
    \textbf{Knowledge Task} \\
    This knowledge describes the geographic relationship between an entity and its country.
    
    \vspace{1em}
    \textbf{Templates $\mathcal{T}$} \\
    \vspace{-1em}
    \begin{itemize}[leftmargin=*, nosep]
        \item \{subject\} is located in \{object\}.
        \item \{subject\} can be found in the nation of \{object\}.
        \item \{subject\} is a landmark situated in \{object\}.
        \item The country associated with \{subject\} is \{object\}.
        \item \{subject\} is a site within \{object\}.
    \end{itemize}
    
    \vspace{1em}
    \textbf{Knowledge Pairs $\mathcal{D}$} \\
    \vspace{-1em}
    \begin{itemize}[leftmargin=*, nosep]
        \item (`bent county high school', `united stated')
        \item (`Opachychi', `ucrania')
    \end{itemize}
        
    \vspace{1em}
    \textbf{Generated Sentences} \\
    \vspace{-1em}
    \begin{itemize}[leftmargin=*, nosep]
        \item "bent county high school" is located in "united stated".
        \item "Opachychi" is located in "ucrania".
        \item "bent county high school" can be found in the nation of "united stated".
        \item "Opachychi" can be found in the nation of "ucrania".
        \item "bent county high school" is a landmark situated in "united stated".
        \item "Opachychi" is a landmark situated in "ucrania".
        \item The country associated with "bent county high school" is "united stated".
        \item The country associated with "Opachychi" is "ucrania".
        \item "bent county high school" is a site within "united stated".
        \item "Opachychi" is a site within "ucrania".
    \end{itemize}
    
    \end{promptbox}
\end{figure}

The first block of the expert is responsible for identifying the next word in the template $\tau$. Its set of input-output pairs consists of $(v_i(\tau), e_i(\tau))$, where $v_i(\tau) \in \mathbb{R}^{L|\mathcal{T}|}$ is the representation of the current template $\tau \in T(\underline{k})$ at position $i$, and $e_i(\tau)$ is the embedding of the next token in the template. Since there are at most $L$ positions across $|T(\underline{k})|$ templates, from Lemma \ref{memorize}, this first block requires at most $L|T(\underline{k})|$ neurons to memorize these input-output pairs.

The second block is responsible for retrieving dictionary values. We use Lemma \ref{memorize} again, where the inputs are the subject representation from $a_2(\underline{w})$, and the target outputs are vector representation of the corresponding objects. Since there are at most $N$ entries in the dictionary $\Delta_{\underline{k}}$, Lemma \ref{memorize} guarantees this mapping can be memorized using $N$ neurons. By using separate blocks for template processing and dictionary retrieval, the total number of neurons required for the $\underline{k}$-th expert is bounded by $L|T(\underline{k})| + N$.

\textbf{Output logits and probability bound.} Finally, we map the output $z'(\underline{w})$ of the $\underline{k}$-th expert to a probability distribution over the vocabulary $\Sigma$.  To construct the weights $W_U$ of the linear unembedding layer, which maps $z'(\underline{w})$ to the logits $l$, we follow the approach used in the proof of Theorem \ref{sparse-circuits} above, using Lemma \ref{unembedding}. 

When applying Lemma \ref{memorize} to construct the experts, we choose the weights of the $\underline{k}$-th expert so that their outputs lie entirely within a $2$-dimensional subspace of $\mathbb{R}^{L|\mathcal{T}|}$. Using the argument from the proof of Theorem \ref{sparse-circuits} above, the linear unembedding matrix $W_U \in \mathbb{R}^{n \times 2L|\mathcal{T}|}$ yields a mapping between these $2$-dimensional subspaces, so that the logit $l_w$ for the valid next token $w \in V(\underline{w})$ is strictly greater than the logit $l_y$ for any invalid token $y \notin V(\underline{w})$.

To ensure that the model has error at most $\epsilon$, we scale the weight matrix $W_U$ by a sufficiently large constant $\gamma > 0$. Using the argument from the proof of Theorem \ref{sparse-circuits} above, for any $\epsilon > 0$, there exists a scaling factor $\gamma$ such that the probability distribution satisfies $\sum_{w \in V(\underline{w})} P(w) > 1 - \epsilon$. This concludes the proof of the theorem. \end{proof}

\section{Appendix: Experiment Details}
\label{Appendix-B}
We conduct all experiments on PyTorch using a single NVIDIA GeForce RTX 4090 GPU with 24GB memory, with the total training time taking less than 10 hours, and report the average results over three random seeds. For data generation, we construct $200$ sentences per template (see Figure~\ref{fig:example} for an illustration). For models trained with router supervision or load balancing, we set the loss weight to $\lambda = 0.2$. The transformer backbone uses a model dimension of $30$, feedforward dimension $10$, and $10$ attention heads, with a single transformer layer in all configurations. Models are trained with a batch size of $64$ for $50$ epochs using the Adam optimizer with an initial learning rate of $0.006$ and exponential decay factor $0.98$ per epoch. Full numerical results are provided in Table~\ref{tab:performance}.

\end{document}